\documentclass[twoside]{article}
\usepackage[accepted]{aistats2019}
\usepackage[utf8]{inputenc} % allow utf-8 input
\usepackage[T1]{fontenc}    % use 8-bit T1 fonts
\usepackage{hyperref}       % hyperlinks
\usepackage{url}            % simple URL typesetting
\usepackage{booktabs}       % professional-quality tables
\usepackage{amsfonts}       % blackboard math symbols
\usepackage{nicefrac}       % compact symbols for 1/2, etc.
\usepackage{microtype}      % microtypography
\usepackage{amsthm}
\usepackage{natbib}
\usepackage{amssymb, amsmath, mathrsfs, dsfont}
\usepackage{mathtools}
\usepackage[swedish,english]{babel}
\usepackage[usenames, dvipsnames]{color}
\usepackage{graphicx}
\usepackage[ruled,linesnumbered]{algorithm2e}
\usepackage{setspace}
\usepackage{mathtools}
\usepackage{subcaption}
\usepackage[draft]{fixme}
\fxsetup{inline,nomargin,theme=color}
\usepackage{soul}

\usepackage{xr}
\DeclarePairedDelimiterX{\infdivx}[2]{(}{)}{%
  #1\;\delimsize\|\;#2%
}

\newcommand{\dsuppeps}[2]{d^\epsilon_{\supp}\infdivx{#1}{#2}}
\newcommand{\dsuppepsk}[3]{d^{#3,\epsilon}_{\supp}\infdivx{#1}{#2}}
\newcommand{\dsuppx}[3]{{#1}_{\supp}\infdivx{#2}{#3}}
\newcommand{\underbracet}[2]{\underbrace{#1\vphantom{\sum}}_{\mbox{\scriptsize #2}}}

\newtheorem{thmlem}{Lemma}

\newtheorem{thmprop}{Proposition}
\newtheorem{thmthm}{Theorem}

\newtheorem{thmasmp}{Assumption}

\newtheorem{thmcond}{Condition}

\theoremstyle{definition}
\newtheorem{thmex}{Example}
\newtheorem{thmdef}{Definition}
\newtheorem{thmrem}{Remark}

\newenvironment{thmproofsketch}[1][Proof sketch]{\begin{trivlist}
\item[\hskip \labelsep {\textit{#1.}}]}{\end{trivlist}}

\makeatletter
\newtheorem*{rep@theorem}{\rep@title}
\newcommand{\newreptheorem}[2]{%
\newenvironment{rep#1}[1]{%
 \def\rep@title{#2 \ref{##1} (Restated)}%
 \begin{rep@theorem}}%
 {\end{rep@theorem}}}
\makeatother

\newreptheorem{theorem}{Theorem}
\newreptheorem{lemma}{Lemma}
\newreptheorem{col}{Corollary}
\newreptheorem{def}{Definition}

\def\cA{\mathcal A}
\def\cL{\mathcal L}
\def\cD{\mathcal D}

\def\cX{\mathcal X}
\def\cY{\mathcal Y}
\def\cF{\mathcal F}
\def\cG{\mathcal G}
\def\cH{\mathcal H}

\def\cZ{\mathcal Z}
\def\cF{\mathcal{F}}

\def\supp{\textnormal{supp}}

\def\hdh{\cH\Delta\cH}

\DeclareMathOperator*{\minimize}{\;\;minimize\;\;}

\def\En{\mathbb{E}}
\DeclareMathOperator*{\E}{\mathbb{E}}

\DeclareMathOperator*{\pr}{\textnormal{Pr}}

\def\ps{p_s}
\def\pt{p_t}
\def\hps{\hat{p}_s}
\def\hpt{\hat{p}_t}

\begin{document}

\twocolumn[

\aistatstitle{Support and Invertibility in Domain-Invariant Representations}

\aistatsauthor{Fredrik D. Johansson \And  David Sontag \And Rajesh Ranganath }

%\aistatsaddress{CSAIL \& IMES, MIT \And  CIMS \& CDS, NYU \And CSAIL \& IMES, MIT}
\aistatsaddress{MIT \And MIT \And  NYU }
]

\begin{abstract}
Learning domain-invariant representations has become a popular approach to unsupervised domain adaptation and is often justified by invoking a particular suite of theoretical results. We argue that there are two significant flaws in such arguments. First, the results in question hold only for a fixed representation and do not account for information lost in non-invertible transformations. Second, domain invariance is often a far too strict requirement and does not always lead to consistent estimation, even under strong and favorable assumptions. In this work, we give generalization bounds for unsupervised domain adaptation that hold for any representation function by acknowledging the cost of non-invertibility. In addition, we show that penalizing distance between densities is often wasteful and propose a bound based on measuring the extent to which the support of the source domain covers the target domain. We perform experiments on well-known benchmarks that illustrate the short-comings of current standard practice.
\end{abstract}

%
% INTRODUCTION
%

\section{Introduction}
Domain transfer is a critical component of many machine learning problems: Self-driving cars must be robust to changes in weather conditions and landscape; Estimates of the efficacy of drugs that pass clinical trials should be valid for the population to which the drugs are prescribed; Policies for robotic control learned in simulated environments should be useful in the real world. In so-called \emph{unsupervised domain adaptation}, labeled data are available only in a limited setting (e.g. driving only in San Francisco; patients restricted to a clinical trial cohort; simulated environments) called the \emph{source domain}. The context in which models are ultimately applied is called the \emph{target domain}.

When the label function is assumed stationary and source and target domains share statistical support, the classical solution to domain adaptation problems is \emph{importance sampling} (IS)~\citep{shimodaira2000improving}. In IS methods, the influence of an observation on the learning algorithm is determined by its likelihood ratio between target and source domains. While asymptotically unbiased, IS estimators suffer from large variance~\citep{cortes2010learning} and are inapplicable when the target domain is not covered by the source. The latter is typical for many of the high-dimensional problems addressed in modern machine learning.

\emph{Domain-invariant representations} have emerged as new, widely-used tools for domain transfer~\citep{ben2007analysis,ganin2016domain,long2015learning} in problems where the label function is assumed fixed, but the covariate distribution changes between domains---so-called \emph{covariate shift}. These methods work by uncovering predictive components of data that are distributed similarly across domains---an idea that has been justified by a string of theoretical work~\citep{ben2007analysis,mansour2009domain,ben2010theory,cortes2011domain}.  Crucially, these bounds do not rely on common support. Related ideas have been applied also under target (label) shift~\citep{gong2016domain}.

Given the prevalence of algorithms learning domain-invariant representations, we ask: Under what conditions do these algorithms recover an optimal hypothesis? What are potential failure modes? We argue that there is discord between existing theoretical guarantees, how they are used to justify learning algorithms, and how these algorithms perform empirically. In particular, we give small example in which a) the objective of many algorithms is minimal but the target error is arbitrarily bad, and b) empirical performance is good but generalization bounds are surprisingly large.

First, we argue that regularizing representations to be domain invariant is too strict, in particular when domains (partially) overlap. We support this claim by giving examples where empirical risk minimization on source data only outperforms domain-invariant representation learning algorithms. As an alternative, we give a generalization bound that measures the lack of \emph{overlapping support} between domains. Our bound applies directly to learned representations and is tight when source and target domains are equal.

Second, for domain-invariant representation learning to succeed, the label must be predictable from the learned representation. When representations are regularized to reduce domain discrepancy, the class of admissible hypotheses shrinks and predictions worsen. This phenomenon may be asymmetric: a representation may be more suitable for the source domain than the target domain. We use this insight to characterize the unobservable adaptation error  from losing information in non-invertible representations.

Finally, we study the performance of domain-invariant representation learning on a well-known benchmark task through the lens of our theoretical findings.

%
% BACKGROUND
%
\section{Background}
We study \emph{unsupervised domain adaptation}, defined as follows. Samples $\cD_{s} = \{(x_i, y_i)\}_{i=1}^n$ of features $X \in \cX$ and labels $Y \in \cY$ are observed from a \emph{source domain}, distributed according to a density $\ps(X, Y)$. In addition, we observe \emph{unlabeled} samples $\cD_{t} = \{x_i'\}_{i=1}^m$ from a \emph{target domain}, distributed according to a density $\pt(X)$. Unobserved labels in the target domain are distributed according to $\pt(Y\mid X)$. Based on $\cD_s$ and $\cD_t$, the unsupervised domain adaptation problem is to obtain an hypothesis $h \in \cH$ that minimizes the \emph{target risk} $R_t$ as measured by a loss function $\ell : \cY \times \cY \rightarrow \mathbb{R}$,
\begin{equation}
R^{\ell}_t(h) \coloneqq \E_{x, y\sim \pt}[\ell(h(x), y)] ~.
\label{eq:targetrisk}
\end{equation}

Analogously to \eqref{eq:targetrisk}, we define the \emph{source risk} as $R^{\ell}_s(h) = \E_{x, y\sim \ps}[\ell(h(x), y)]$. When clear from context, we leave out the superscript $\ell$ indicating the loss function. In the sequel, unless otherwise stated, we let $\cY = \{0,1\}$ and $\ell$ be the zero-one loss, $\ell(y,y') = \mathds{1}[y \neq y']$. We call $R_t(h) - R_s(h)$ the \emph{adaptation error}.

In this work, we make the \emph{covariate shift} assumption which states that the conditional density of labels given features is stationary across domains. This is justifiable in some problems but not in all~\citep{gong2016domain}.
\begin{thmasmp}[]\label{asmp:cshift}
  Domains $\ps(X,Y)$ and $\pt(X,Y)$ satisfy the \textit{covariate shift} assumption if
  $$
  \ps(Y\mid X) = \pt(Y\mid X) = p(Y\mid X)~.%,\;\; \ps(X) \neq \pt(X)~.
  $$
\end{thmasmp}
%\fxfatal{Removed $\ps(X) \neq \pt(X)$. ``Covariate shift'' doesn't make sense then though...}

We say that $Y$ is \emph{realizable} in $\cH$ if $p(Y\mid X) \in \cH$ and that $Y$ is \emph{identifiable} over $p_t$ if, under a set of assumptions on $p_s, p_t$, a function $h$ may be obtained based on knowledge of $\ps(Y, X)$ and $\pt(X)$ such that $\forall x \in \supp(\pt) : h(x) = p(Y\mid X = x)$. In the case of deterministic hypotheses and labels, or when only label expectations are of interest, we may substitute conditional densities with appropriate mappings.

As no labels are observed from the target domain, models that minimize risk (only) on the source domain are often biased. There are two common alternative strategies to minimize target risk: \emph{importance-weighting} and \emph{minimization of upper bounds on the target risk}. %While the latter is the main concern of this work, we describe importance weighting below to contextualize previous and new results.

\subsection{Importance weighting}

Under Assumption~\ref{asmp:cshift} (covariate shift), the target risk $R_t$ may be approximated using importance-weighted samples $\cD_s$ from the source density~\citep{shimodaira2000improving},
\begin{align}
%R_t(h) & = \E_{x,y \sim p_s}\left[\frac{p_t(x)}{p_s(x)} \ell(h(x),y) \right] \\
%& \approx \frac{1}{n} \sum_{i=1}^n \frac{p_t(x_i)}{p_s(x_i)} \ell(h(x_i), y_i)
\hat{R}_s^w(h) := \frac{1}{n} \sum_{i=1}^n w(x_i) \ell(h(x_i), y_i)~. \label{eq:importancefinite}
\end{align}
%$$
%\E_{x\sim q(x)}[f(x)] = \E_{x\sim p(x)}\left[\frac{q(x)}{p(x)}f(x) \right]~.
%$$
If the weighting function $w$ is chosen to be $w(x) = p_t(x)/p_s(x)$, $\hat{R}_s^w(h)$ is a consistent estimator of $R_t(h)$ under the following assumption.
\begin{thmasmp}[Sufficient support]
  We say that $\ps$ has $\epsilon$-\textit{sufficient support} for $\pt$ if $\forall x \in \supp(\pt) : \ps(x) \geq \epsilon$, with $\epsilon > 0$. This is also called $\epsilon$-\textit{overlap}.
  \label{asmp:supp}
\end{thmasmp}
\citet{cortes2010learning} give generalization bounds for importance-weighted estimates such as~\eqref{eq:importancefinite}. These estimates have high variance when the largest $\epsilon$ in Assumption~\ref{asmp:supp} is small;  if there is no such $\epsilon > 0$, importance weighting is inapplicable without modification.
%This is equivalent to the support of $\pt$ not being covered by that of $\ps$.
%Finally, $p_s$ and $p_t$ are rarely known and need to be estimated from data.
%Instead we may attempt to bound the target risk using observed data.

\subsection{Upper bounds on target risk}
%To mitigate this, a common strategy is to control the way hypotheses extrapolate  beyond the support of the source domain.
When the source domain does not provide sufficient support, $\supp(\pt) \not\subseteq \supp(\ps)$, the target risk of a learned hypothesis may not be consistently estimated without further assumptions. However, we may bound the target risk from above and minimize this bound.

\citet{ben2007analysis} introduced the $\cH\Delta\cH$-distance to measure the worst-case loss from extrapolating between domains using binary hypotheses in a class $\cH$. Let $R^\ell_p(h, h') \coloneqq \E_{x \sim p}[\ell(h(x), h'(x))]$ denote the expected disagreement between two hypotheses $h, h'$. Then, the $\hdh$-distance between $\ps$ and $\pt$ is\footnote{The definition is sometimes given with a factor 2. %We have chosen to leave this out and adjust subsequent results.
}
\begin{align}%
d_{\cH\Delta\cH}(\ps,\pt) \coloneqq \sup_{h, h' \in \cH} \left|R_s(h,h') - R_t(h,h') \right|~. \label{eq:dhdh}
\end{align}

By reducing the model class $\cH$, the potential disagreement $d_{\cH\Delta\cH}$  between member functions may be reduced---as well as the capacity of $\cH$ to predict the label $Y$. The best-in-class joint hypothesis risk is
\begin{equation}
\lambda_{\cH} \coloneqq \inf_{h\in \cH} [R_s(h) + R_t(h)]~.
\end{equation}

These quantities lead to the following bound by applying the triangle-inequality of classification error.

\begin{thmthm}[Adaptation bound by~\citet{ben2010theory}]\label{thm:benadapt}
  Under Assumption~\ref{asmp:cshift}, for all $h \in \cH$,
  \begin{equation}\label{eq:benadapt}
  R_t(h) \leq R_s(h) + d_{\hdh}(\ps, \pt) + \lambda_{\cH}~.
  \end{equation}
\end{thmthm}

The result \eqref{eq:benadapt} may be bounded further based on the risk on a  sample $(x_1, y_1), \ldots, (x_n, y_n) \sim \ps(x,y)$ from the source domain, and an empirical estimate of $d_{\hdh}(\ps, \pt)$~\citep{ben2010theory}. Similar results have also been obtained for continuous labels~\citep{mansour2009domain,cortes2011domain}.
%We note that for $\ps = \pt$ and $h = \argmin_{h'\in \cH} [R_s(h') + R_t(h')]$, the bound is loose by a factor 3.

\subsection{Domain-invariant representation learning}
Theorem~\ref{thm:benadapt}, and a suite of follow-up work, have been used to justify algorithms based on learning \emph{domain-invariant representations}---transformations of features such that the source and target domains are approximately indistinguishable in the transformed space~\citep{ben2007analysis}. We describe these below.

Let a random variable $Z \in \cZ$ be a representation of the input features $X$, parameterized by a \emph{deterministic} function $\phi(X) =: Z$ with $\phi \in \cG \subset \{\cX \rightarrow \cZ\}$. Hypotheses $h\in \cH$ for $Y$ are formed by compositions $h = f \circ \phi$ with prediction functions $f \in \cF \subset \{\cZ \rightarrow \cY\}$ operating in the representation space $\cZ$, and $\cH \coloneqq \{f \circ \phi : f\in \cF, \phi \in \cG\}$. The probability of a set $\mathbf{z} \subseteq \cZ$ induced by $\phi$ is then
$p(Z \in \mathbf{z}) = \int_{x \in X} p(X = x) \mathds{1}[\phi(x) \in \mathbf{z}]dx$. If $\phi$ does not induce atoms, $p(Z)$ is a density; we consider only this case in the sequel.

We say that a representation $Z_\phi \coloneqq \phi(X)$ of $X$ is \emph{domain-invariant} if $\ps(Z_\phi) = \pt(Z_\phi)$.  A common approach to learning approximately domain-invariant representations is to solve the following  problem\footnote{We leave out additional regularization of $\phi$ and $f$.}.
\begin{equation}
\minimize_{\phi \in \cG, f \in \cF} \underbracet{\hat{R}_s(f \circ \phi)}{Source risk} \;+\; \alpha \cdot \underbracet{d(\hps(Z_\phi), \hpt(Z_\phi))}{Domain variance in $Z$}
\label{eq:discmin}
\end{equation}
Here, $\hps, \hpt$ denote empirical distributions of $\ps$ and $\pt$, $d$ is a distance function on densities, and $\alpha$ is a hyperparameter. In the next section, we describe several instantiations of \eqref{eq:discmin}.

\section{Related work}
\label{sec:related}
Domain adaptation has been studied primarily under Assumption~\ref{asmp:supp} (covariate shift)~\citep{pan2010survey}, which is the setting also of this work. However, prediction under shift in the target $p(Y)$ and conditional $p(X\mid Y)$ has also been considered~\citep{zhang2013domain,gong2016domain,lipton2018detecting}.
A common approach in both settings is to learn representations or projection of observed data that is invariant to the shift in question by minimizing adversarial  losses~\citep{ganin2015unsupervised,bousmalis2016domain,tzeng2017adversarial}, integral probability metrics such as the maximum mean discrepancy (MMD)~\citep{pan2011domain,long2015learning,long2016deep,baktashmotlagh2013unsupervised} and the Wasserstein distance~\citep{shalit2016estimating,courty2017joint}, or other divergences~\citep{berisha2016empirically,si2010bregman,muandet2013domain}.

Many recent methods attempt to solve domain adaptation under covariate shift by optimizing objectives similar to \eqref{eq:discmin}~\citep{ganin2015unsupervised,long2015learning,bousmalis2016domain}, with the distance $d$ chosen to be a metric such that $d(p,q) = 0 \mbox{ iff } p=q$. However, \citep{gong2016domain} point out that it is not clear under what conditions $\ps(\phi(X)) \approx \pt(\phi(X))$ would imply $\ps(Y \mid \phi(X)) \approx \pt(Y \mid \phi(X))$. \citet{ben2010impossibility} showed that Assumption~\ref{asmp:cshift} (covariate shift) and small $d_{\cH\Delta\cH}$ are not sufficient on their own to identify $Y$. On the other hand, \citet{ben2012hardness} showed that Assumption~\ref{asmp:supp} (sufficient support) is sufficient by counterexample through a reduction of the Left-Right problem~\citep{kelly2010universal}. \citet{ben2014domain} subsequently gave both upper and lower learning bounds for nearest-neighbor learners under Assumptions~\ref{asmp:supp} and so-called probabilistic Lipschitzness.

Next, we argue that that searching for a representation $Z$ such that $\ps(Z) \approx \pt(Z)$ is often undesirable and that objectives like that in \eqref{eq:discmin} are insensitive to information lost in domain-invariant representations. \footnote{In work prepared concurrently with the original publication of this work, \citet{zhao2019learning} make similar observations about deficiencies in the existing literature on unsupervised domain adaptation, and provide additional theoretical and experimental results.}

%
% REPRESENTATION LEARNING
%
\section{Limitations of domain-invariant representation learning}
\label{sec:representations}
In this section, we give concrete examples of the failure modes of domain-invariant representation learning, and propose a shift in focus for future research.

\subsection{Representation-induced adaptation error}
\label{sec:invertibility}
When features $X$ are high dimensional, they often contain information that is redundant or irrelevant for predicting the label $Y$ but distinguishes the source and target domains; the higher the dimensionality, the less likely overlap is to hold in $X$~\citep{d2017overlap}. The adaptation bounds reviewed in the previous section suggest that removing such information may reduce the difference between source and target risk by making domains closer in density. However, doing so may also introduce an unobservable error, as we see this in the following example.

\begin{figure}[t!]
  \centering
  \includegraphics[width=.8\columnwidth]{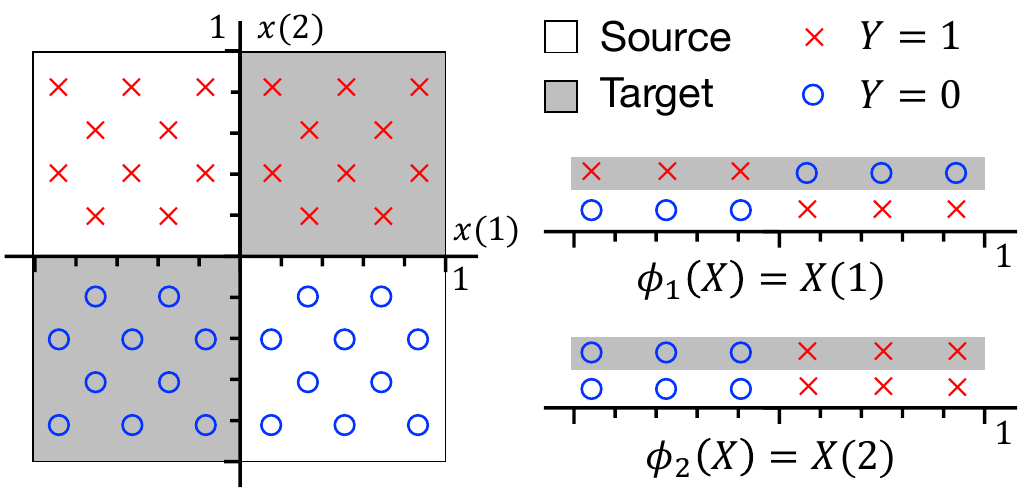}
  \caption{\label{fig:checker} Illustration of Example~\ref{ex:checker} in which there are two optimal solutions to \eqref{eq:discmin} with objective value 0 but with radically different target risk.}
\end{figure}

\begin{thmex}[Variable selection]\label{ex:checker}
Let $\cX = [-1,1]^2$ with $\pt(x) = .5$ if $x$ is in the lower left or upper right quadrants, $\{[0, 1] \times [0,1] \cup [-1, 0]\times[-1,0]\}$, and $\ps(x) = .5$ if $x$ is in the upper left or lower right quadrants, $\{[-1, 0] \times [0,1] \cup [0, 1]\times[-1,0]\}$. Further, let $Y=1$ if $x(2) > 0$, and $0$ otherwise (see Figure~\ref{fig:checker}).
Now, let $\cG$ be the set of variable selections from $\mathbb{R}^2$ to $\mathbb{R}$ and let $\cF$ be the set of threshold functions in $\mathbb{R}$. Then, for either selection of a single variable, $\phi_1(x) = x(1)$ or $\phi_2(x) = x(2)$, with $Z=\phi_i(X)$ we have that $\ps(Z) = \pt(Z)$ and $d_{\cF\Delta\cF}(\ps(Z), \pt(Z)) = 0$, and the function $f(z) = \mathds{1}[z > 0]$ has $R_s(f \circ \phi_1) = R_s(f\circ \phi_2) = 0$. However, $R_t(f \circ \phi_2) = 0$, but $\forall f\in \cF : R_t(f \circ \phi_2)\geq 1$. Hence, objective \eqref{eq:discmin} is uninformative of the target risk.
\end{thmex}

Example~\ref{ex:checker} illustrates the impossibility of domain adaptation without overlap or other additional assumptions. Based on the observed data, there is nothing that distinguishes a failure case with maximum target risk from a successful case with minimal target risk. This is true even despite the fact that the problem satisfies the following strong condition.

\begin{thmasmp}[Optimal domain-invariant representation]
There exist a representation $\phi \in \cG$ and $f \in \cF$ such that $\forall x \in \supp_X(\ps) \cup \supp_X(\pt) : f(\phi(x)) = p(Y \mid X=x)$ and $\ps(\phi(X)) = \pt(\phi(X))$~.
\label{asmp:domaininv}
\end{thmasmp}

Assumption~\ref{asmp:domaininv} is by no means guaranteed to hold in practice. Often, variables that are distributed differently across domains are critical for prediction. Regardless, Assumption~\ref{asmp:domaininv} is \emph{necessary} for domain-invariant representation learning to be consistent. However, as strong as this assumption is, it is not \emph{sufficient} for consistent domain adaptation---not with domain-invariant representations nor with any other method.

Even in problems that are possible to solve consistently, a learned representation may be more predictive on the source domain than the target domain. To reason about this case, we must apply Theorem~\ref{thm:benadapt} to the hypothesis space $\cH_\phi = \{f \circ \phi : f \in \cF\}$ induced by the representation $\phi$.

Then, for all $f\in \cF$,
\begin{equation}
R_t(f\circ \phi) \leq R_s(f\circ \phi) + d_{\cF\Delta\cF}(\ps(Z), \pt(Z)) + \lambda_{\cH_\phi}.
\label{eq:phibound}
\end{equation}
Here, $R_s(f\circ \phi)$ and $d_{\cF\Delta\cF}(\ps(Z), \pt(Z))$ may be bounded and minimized but, in contrast, \emph{$\lambda_{\cH_\phi}$ is unobserved and may increase when solving~\eqref{eq:discmin}}.

\begin{thmprop}\label{prop:phibound}
  For all $\phi \in \cG, f \in \cF$ as defined above, we have with $Z = \phi(X)$ and $\cH_\phi = \{f \circ \phi : f\in \cF\}$
  \begin{align}
    d_{\cF\Delta \cF}(\ps(Z), \pt(Z)) & \leq d_{\cH\Delta \cH}(\ps(X), \pt(X)) \\
    \lambda_{\cH_\phi} & \geq \lambda_\cH~.
  \end{align}%
\end{thmprop}
\begin{proof}
  The results follow immediately from the definitions of $d_{\cH\Delta\cH}$ and $\lambda_\cH$, that $d_{\cF\Delta \cF}(\ps(Z), \pt(Z)) = d_{\cH_\phi\Delta\cH_\phi}(\ps(X), \pt(X))$, and that $\cH_\phi \subseteq \cH$.
\end{proof}
As a result of Proposition~\ref{prop:phibound}, solving \eqref{eq:discmin} implies neither minimization of the RHS of \eqref{eq:benadapt} or \eqref{eq:phibound}.
One interpretation of this result, and of Example~\ref{ex:checker}, is that covariate shift  (Assumption~\ref{asmp:cshift}) need not hold with respect to the representation $Z = \phi(X)$, even if it does with respect to $X$. With $\phi^{-1}(z) = \{x : \phi(x) = z\}$,
$$
p_t(Y\mid z) = \frac{\int_{x\in \phi^{-1}(z)}p(Y\mid x)p_t(x) dx}{\int_{x \in \phi^{-1}(z)}p_t(x)dx} \neq p_s(Y\mid z)~.
$$
Equality holds for general $\ps, \pt$ only if $\phi$ is invertible. In Section~\ref{sec:support}, we define a quantity that measures the effect of this discrepancy and how it relates to invertibility, and use it to bound the target risk.

We summarize this section in a statement inspired by Lemma~2 in \citet{bareinboim2013general}.

\noindent\fbox{\parbox{0.97\columnwidth}{
If there are two distinct hypotheses $h, h'$ for the label $Y$ that are both consistent with $\ps(X, Y)$ and $\pt(X)$ and a set of assumptions $\cA$, but result in different predictions on $\pt(X)$, $Y$ is not identifiable over $\pt(X)$.
}}

Like causal inference~\citep{pearl2009causality}, successful domain adaptation is often entirely reliant on making appropriate assumptions about unobservable quantities.

\subsection{The cost of domain invariance}

\begin{figure*}[tbp!]
  \centering
  \begin{subfigure}{.31\textwidth}
    \centering
    \includegraphics[width=1\textwidth]{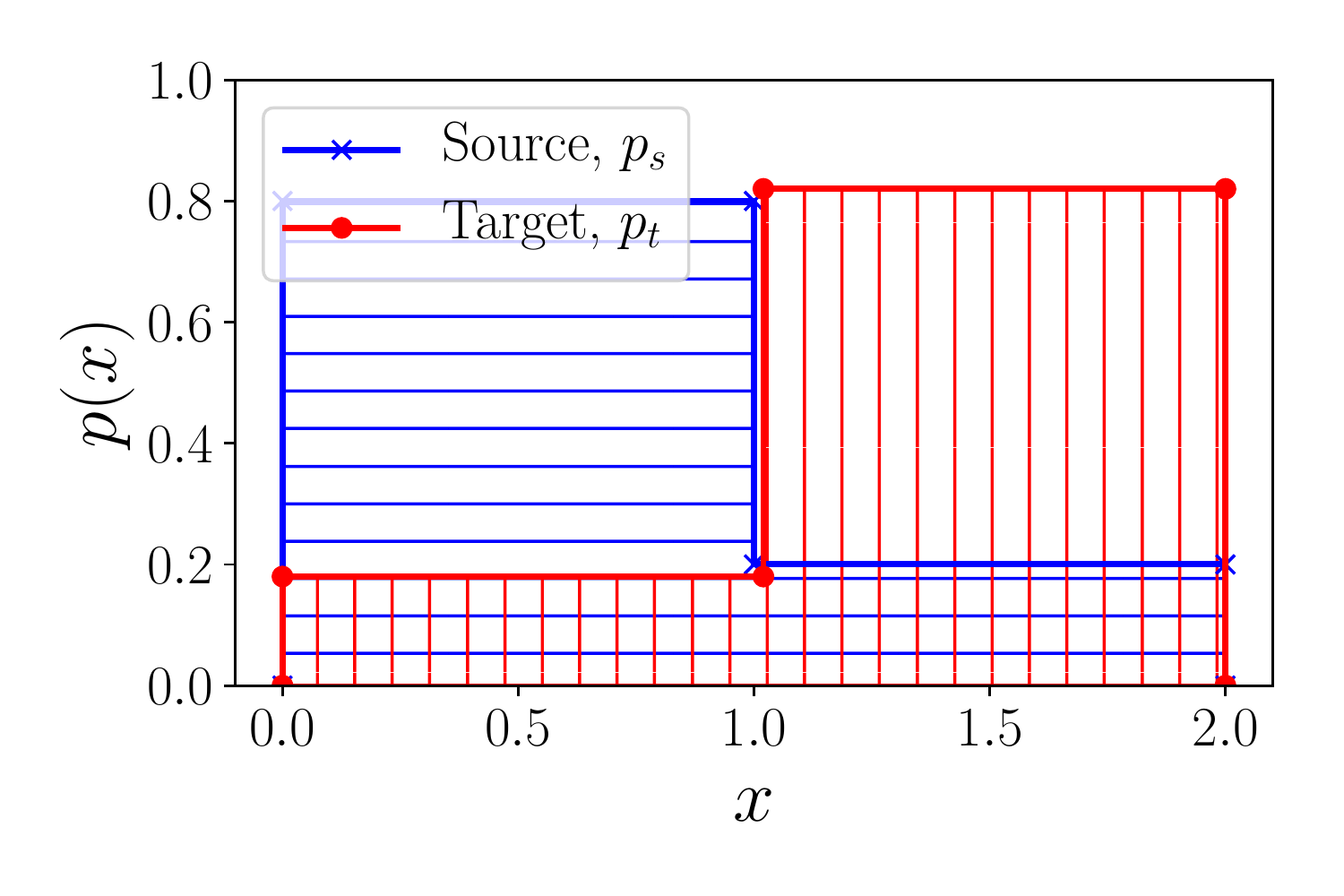}
    \caption{\label{fig:loose_example_A}Problem A}
  \end{subfigure}
  \begin{subfigure}{.31\textwidth}
    \centering
    \includegraphics[width=1\textwidth]{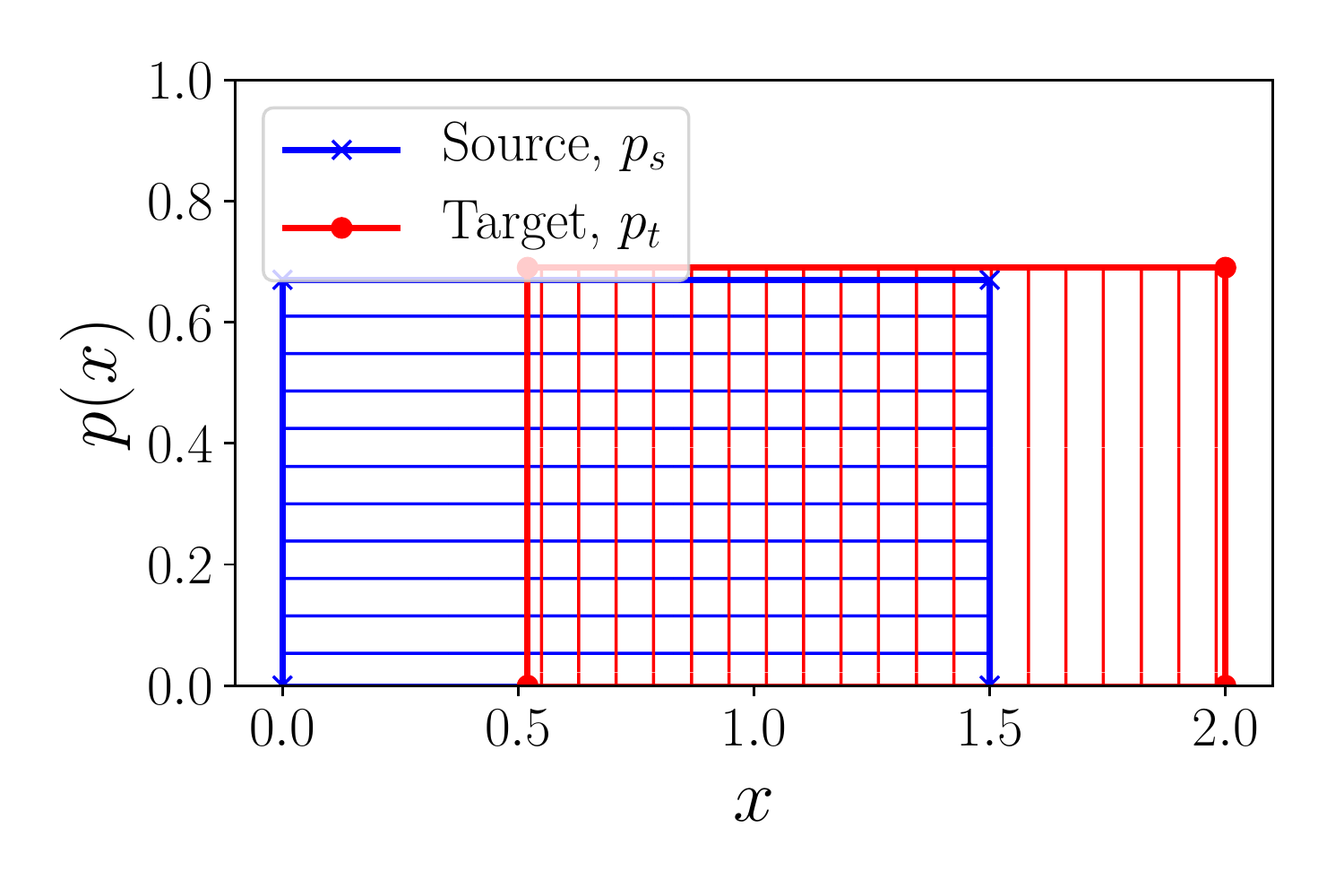}
    \caption{\label{fig:loose_example_B}Problem B}
  \end{subfigure}
  \begin{subfigure}{.31\textwidth}
    \centering
    \includegraphics[width=1\textwidth]{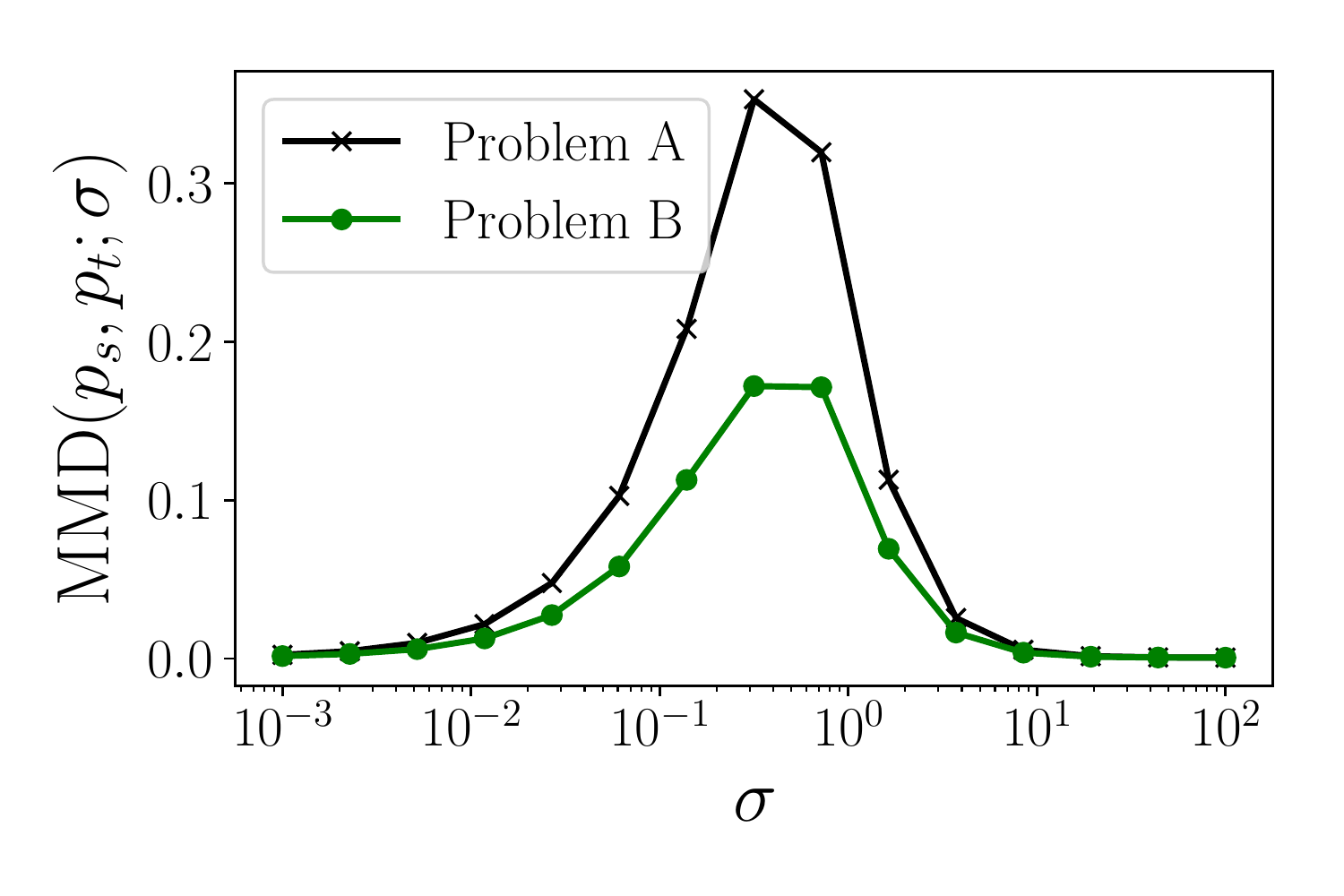}
    \caption{MMD for varying bandwidth, $\sigma$}
  \end{subfigure}
  \caption{\label{fig:loose_example} Examples illustrating the counter-intuitive effects of using density distance metrics for regularizing domain adaptation methods. Despite the fact that sufficient support is satisfied in Problem A, typical adaptation bounds (using e.g. the RBF-kernel MMD, see (c)) are smaller for Problem B than for Problem A. In contrast, our proposed support sufficiency divergence with $\epsilon=0.2$ (see Section~\ref{sec:support}) is 0 in Problem A and 0.33 for Problem B.}
\end{figure*}

\label{sec:costinvariance}
A desired property of adaptation bounds is that they are as tight as possible when Assumption~\ref{asmp:supp} (sufficient support) holds, since consistent estimation is possible in this setting~\citep{ben2012hardness}\footnote{By ``consistency'', we refer to the convergence of an estimate to a quantity of interest given enough samples.}. However, bounds based on Theorem~\ref{thm:benadapt} do not always have this property, and their looseness is often independent of the observed risk on the source domain. We give an example of how unintuitive this can be below.

\begin{thmex}
We illustrate two examples of source and target densities in Figure~\ref{fig:loose_example} along with the estimated maximum mean discrepancy (MMD)~\citep{gretton2012kernel} between domains for a Gaussian RBF-kernel with varying bandwidth $\sigma$. The MMD has been used to bound $d_{\cH\Delta\cH}$ and the target risk  in~\citet{gretton2009covariate,long2015learning,pan2011domain,gong2016domain}, among others. Despite there being a significant \emph{lack of overlap} between the support of source and target domains in Problem B, the MMD is smaller than in Problem A, in which the support of the target domain is completely covered by the source density. However, Problem A satisfies sufficient assumptions for identifiability, whereas Problem B does not. This illustrates a drawback of representation learning methods that penalize distributional distance between domains.
\label{ex:overlap}
\end{thmex}

The problem illustrated in Example~\ref{ex:overlap} has practical consequences, as we see in Section~\ref{sec:exp_invariance}. When label marginal distributions differ in a classification task, but domains partially overlap, requiring domain invariance is often too strict. In fact, in our examples, training using only source labels often does better than domain-invariant representation learning.

\section{A new support-based bound}
\label{sec:support}
We proceed to bound the target risk of an hypothesis in terms of its error on the source domain and the expected lack of sufficient support. This bound is aimed at overcoming limitations of existing bounds by a) explicitly characterizing the risk induced by non-invertible representations and b) avoiding unnecessary side effects of domain invariance.

We say that there is lack of \emph{sufficient support} at a point $x$ if the target density is larger than the source density and the source density is small, as defined by $\delta_{p,q}(x)$,
\begin{equation}
\delta_{p,q}(x) = \mathds{1}[q(x) \geq p(x) \;\mbox{ and }\; p(x) < \epsilon]~. \label{eq:supplack}
\end{equation}
We let $\delta_{s,t}(x)$ serve as short-hand for $\delta_{\ps, \pt}(x)$. Below, we define the \emph{support sufficiency divergence}.
\begin{thmdef} For distributions, $p, q$, the \emph{support sufficiency divergence} from $p$ to $q$ is defined by
$$
\dsuppeps{p}{q} \coloneqq \E_{q}[\delta_{p,q}(x)] - \E_{p}[\delta_{p,q}(x)]
$$
\end{thmdef}

Note that $d^\epsilon_{\supp}$ is not symmetric, but is $0$ for $p=q$. Crucially however, it is 0 also when $p\neq q$ for some choices of $\epsilon$, if $\supp(p) = \supp(q)$. Further, it holds that $0 \leq d^\epsilon_{\supp} \leq 1$ and the bounds are tight (see Appendix~\ref{app:suppbounds} for a proof).

Our main result builds on the idea that we can expect an hypothesis to be accurate on the target domain in regions where the source density is sufficiently high. First, let $w_{p,q}^\epsilon(x)$ be a weighting function such that
\begin{equation}
w_{p,q}^\epsilon(x) = \left\{
\begin{array}{ll}
  q(x)/p(x) & \mbox{ if } p(x) \geq \epsilon \\
  1 & \mbox{ othewrise }
\end{array}
\right.
\end{equation}

We may state the following result.

\begin{thmlem}\label{lem:overlap}
Let $\ps(x), \pt(x)$ be densities over $\cX$. Further, let $\ell : \cX \rightarrow \mathbb{R}_+$ be a function such that $\exists M > 0 : \forall x\in \cX : \ell(x) \in [0, M)$. Then, with $\epsilon > 0$,
\begin{equation*}
\begin{array}{lll}
\En_{\pt}[\ell(x)]  \leq  \underbracet{\En_{\ps}\left[w_{\ps,\pt}^\epsilon(x) \ell(x)\right]}{\emph{Weighted expectation}}  \vspace{.5em}
 +  \underbracet{M \cdot \dsuppeps{\pt}{\ps}}{\emph{Support discrepancy}} ~.
\end{array}
%\label{eq:overlaplem}
\end{equation*}
Equality holds if $\pt = \ps$ or if $\forall x\in \supp(\pt) : \ps(x) = \epsilon$. The second term is 0 if and only if Assumption~\ref{asmp:supp} holds with $\epsilon \leq \inf_{x : \pt(x) \geq \ps(x)} \ps(x) $, by definition. The proof can be found in Appendix~\ref{app:proof_lem_overlap}.
\end{thmlem}

Before we state our main result, we define a measure of the impact of non-invertibility in representations.

\begin{thmdef}\label{def:infoloss}%
Given are domains $\ps$ and $\pt$, a prediction function $f\in \cF$, a label $Y$, a loss $\ell$ and a representation $Z = \phi(X)$. Let
$$
\Delta_{q,p}(x) := \E_{q(y\mid \phi(x))}[\ell(f(\phi(x)), y)] - \E_{p(y\mid x)}[\ell(f(\phi(x)), y)]
$$
Then, the \emph{excess target information loss} is
\begin{align*}
 \eta^\ell_\phi(f, Y)
& = \E_{\pt(x)}\left[ \Delta_{\pt,p}(x) - \Delta_{\ps,p}(x) \right]
\end{align*}
We say that the information loss induced by the representation $\phi(x)$ is \emph{symmetric} if $\eta^\ell_\phi(f, y) = 0$. Both $\Delta$ and $\eta$ are always 0 for invertible $\phi$. Note also that $\eta$ may be negative, although we don't expect this in practice as we explain later.
\end{thmdef}

By Lemma~\ref{lem:overlap} and Definition~\ref{def:infoloss}, we have the following.

\begin{thmthm}\label{thm:repbound}
Consider any feature representation $z = \phi(x)$ with $\phi \in \cG$ and prediction function $f \in \cF$, and define $h = f \circ \phi$. Further, let $\ps(Z)$ and $\pt(Z)$ be the two distributions induced by the representation $\phi$ applied to $X$ distributed according to $\ps(X), \pt(X)$.  Further, assume that for any hypothesis $h\in \cH$ and a loss function $\ell$, $\sup_{x\in \cX, y\in \cY, h\in \cH}[\ell(h(x),y)] \leq M$. For any $\epsilon > 0$,
\begin{align}\label{eq:mainthm}
R_t(f \circ \phi) & \leq \underbracet{\En_{\ps}\left[w_{\ps, \pt}^\epsilon(z) \ell(f(z), y)\right]}{\emph{Observable}} \\
& + M\underbracet{\dsuppeps{\ps(z)}{\pt(z)}}{\emph{Observable}}  + \underbracet{\eta^\ell_\phi(f, y)}{\emph{Unobservable}}~. \nonumber
\end{align}
\end{thmthm}
\begin{thmproofsketch}
For any $h = f \circ \phi$, we have that
$$
E_{\pt(x,y)}[\ell(h(x),y)] = E_{\pt(z,y)}[\ell(f(z),y)]~.
$$
By adding and subtracting $E_{\pt(z)\ps(y\mid z)}[\ell(f(z),y)]$,
\begin{align*}
& E_{\pt(z,y)}[\ell(f(z),y)] = E_{\pt(z)\ps(y\mid z)}[\ell(f(z),y)] \\
& + E_{\pt(z)\pt(y\mid z)}[\ell(f(z),y)] - E_{\pt(z)\ps(y\mid z)}[\ell(f(z),y)]
\end{align*}
The last two terms equal $\eta^\ell_\phi(h, y)$ as $\ps(y\mid x) = \pt(y\mid x)$ by Assumption~\ref{asmp:cshift}. Note that the marginal density over $z$ is equal in both of the last terms. The first term may be decomposed by the support of $\ps$. With $L_s(z) = E_{\ps(y\mid z)}[\ell(f(z),y)\mid Z=z]$, we get
\begin{align*}
& (*) := E_{\pt(z)\ps(y\mid z)}[\ell(f(z),y)] \\
& = \int_{z : \ps(z) \geq \epsilon} \pt(z) L_s(z) dz
 + \int_{z : \ps(z) < \epsilon} \pt(z) L_s(z) dz
\end{align*}
Adding and subtracting $\int_{z : \ps(z) < \epsilon} \ps(z) L_s(z) dz$, we get
\begin{align*}
& (*) = \En_{\ps}\left[w_{\ps,\pt}^\epsilon(z) \ell(f(z), y)\right] \\
& + \int_{z : \substack{\ps(z) < \epsilon \\ \ps \leq \pt}} \underbrace{(\pt(z) - \ps(z))L_s(z)}_{\geq 0} dz \\
& + \underbrace{\int_{z : \substack{\ps(z) < \epsilon \\ \ps > \pt}} (\pt(z) - \ps(z))L_s(z) dz}_{\leq 0} ~.
\end{align*}
Bounding the second term by $\dsuppeps{\ps}{\pt}$ and removing the third non-positive term, we obtain the result. For a full proof, see Appendix~\ref{app:proof_repbound}.\qed
\end{thmproofsketch}

Theorem~\ref{thm:repbound} is consistent with our intuition that increasing the sufficiency of the support of $\ps$ for $\pt$ leads to better adaptation. If this overlap is increased without losing information, such as through collection of additional samples, this is usually preferable.

Unlike bounds based on the triangle inequality ~\citep{ben2010theory,mansour2009domain,cortes2011domain}, the bound in Theorem~\ref{thm:repbound} is tight when $\ps(X) = \pt(X)$. On the other hand, when the supports of $\ps(Z)$ and $\pt(Z)$ are completely disjoint, the bound is non-informative. In Section~\ref{sec:kernel} we obtain a tighter bound for the disjoint case by incorporating additional assumptions. In many problems, however, there is partial overlap, such as under label marginal shift.

For domains with common and bounded support, $\epsilon$ may be chosen such that minimizing the bound of Theorem~\ref{thm:repbound} reduces to importance sampling. In fact, we may view Theorem~\ref{thm:repbound} as a middle-ground between importance sampling estimates and upper bounds on the target risk, using importance sampling where feasible. The choice of $\epsilon$ in Lemma~\ref{thm:repbound} trades off the sizes of the two middle terms in \eqref{eq:mainthm}---small $\epsilon$, larger first term and vice versa. Additionally, if $\ell(x)$ is 0 everywhere on $\supp(p)$, the first term is 0. The choice of $\epsilon$ also affects the variance in Monte-Carlo estimates of these terms. If $\epsilon$ is close to $0$, the weights $w_{\ps, \pt}^\epsilon(z)$ are potentially larger, and variance increases~\citep{cortes2010learning}.

When $\phi$ is invertible, $\eta^\ell_\phi(f, y) = 0$ as $\pt(y\mid \phi(x)) = \pt(y\mid x)$. \citet{shalit2016estimating} gave a bound based on integral probability metrics in the style of Theorem~\ref{thm:benadapt}, with the additional restriction that $\phi$ is invertible. However, this is a strong restriction as such $\phi$ cannot increase the sufficiency of support w.r.t. $\ps(z)$ and $\pt(z)$. We conjecture that under appropriate assumptions of smoothness, $\eta$ is larger for less invertible $\phi$. By encouraging $\phi$ to be near-invertible, this is mitigated. This would serve as justification for reconstruction losses used by for example~\citet{bousmalis2016domain}. Alternatively, $\eta^\ell_\phi(f, y) = 0$ if any information lost in $\phi$ is equally important for predicting labels in the source domain as in the target. If $\eta = 0$ is always true, Assumption~\ref{asmp:domaininv} is sufficient for identification of the label.

\subsection{Incorporating assumptions on the loss}
\label{sec:kernel}
In Theorem~\ref{thm:repbound}, the loss at points outside of the overlap between domains is bounded from above by a constant, $M$. As a result, the bound is uninformative for disjoint domains. If prior knowledge about the label function is available, we may address this by making assumptions about how the label function extrapolates, akin to Theorem~\ref{thm:benadapt}.
Below, we give an alternative bound based on an assumption that the loss $\ell(f \circ \phi)$ of hypotheses using a representation $\phi$ belongs to a known family $\cL$. Critically, this new bound remains qualitatively different from previous work as a) it penalizes extrapolation between domains \emph{only in regions where the source density is low} and b) it explicitly characterizes the excess target risk due to information lost in the learned representation.

\begin{thmdef}
We define the \emph{integral probability metric (IPM) support sufficiency divergence} between densities $p,q$ on $\cX$ with respect to a class of functions $\cL$ by
\begin{equation}
\dsuppepsk{p}{q}{\cL} \coloneqq \sup_{\ell \in \cL} \left| \E_{q}[\delta_{p}(x)\ell(x)] - \E_{p}[\delta_{p}(x)\ell(x)] \right|
\label{eq:kerneldist}
\end{equation}
where $\delta_{p}(x) = \mathds{1}[p(x) < \epsilon]$.
\end{thmdef}

\begin{thmthm}\label{thm:repboundk}
Assume that for any representation $\phi \in \cG$, and any $f\in \cF$,
$
\E_{\ps(y\mid \phi(x)}[\ell{f(\phi(x), y)}] \in \cL
$.
Under the conditions of Theorem~\ref{thm:repbound}, we have
\begin{align*}
R_t(f \circ \phi) & \leq \underbracet{\En_{\ps}\left[w_{\ps, \pt}^\epsilon(z) \ell(f(z), y)\right]}{\emph{Observable}} \\
& + \underbracet{\dsuppepsk{\ps(z)}{\pt(z)}{\cL}}{\emph{Observable}}  + \underbracet{\eta^\ell_\phi(f, y)}{\emph{Unobservable}}~.
\end{align*}
\end{thmthm}

\begin{thmrem}
Theorem~\ref{thm:repboundk} provides a tighter bound than Theorem~\ref{thm:repbound} at the cost of stronger assumptions.
With $M \geq \sup_{x \in \cX, y\in\cY, \ell\in \cL} \ell(h(x),y)$, and $\ell > 0$,
$$
\dsuppepsk{p}{q}{\cL} \leq M \max\{\dsuppeps{p}{q}, \dsuppeps{q}{p}\} \leq M.
$$
For the first inequality to be tight, the maximizer $\ell^*$ of \eqref{eq:kerneldist} must be flexible enough to always be equal to $M$ when $q>p$ and always equal to $0$ when $q<p$. This is unlikely to be true of the actual loss when the supports of $p$ and $q$ overlap. Instead, it is common to assume that $\ell^*$ obeys some smoothness conditions. In Appendix~\ref{app:proof_kernel} we show how $\dsuppepsk{p}{q}{\cL}$ may be estimated using kernel evaluations if $\cL$ is a reproducing-kernel Hilbert space, following~\citet{gretton2012kernel}.
\end{thmrem}

\section{Empirical results}
\begin{figure}[tbp!]
  \centering
  \vspace{1em}
  \begin{subfigure}{.4\columnwidth}
    \centering
    \includegraphics[width=.6\textwidth]{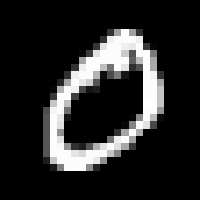}
    \caption{\label{fig:mnist}``0'' in MNIST}
  \end{subfigure}
  \;
  \begin{subfigure}{.4\columnwidth}
    \centering
    \includegraphics[width=.6\textwidth]{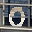}
    \caption{\label{fig:mnistm}``0'' in MNIST-M}
  \end{subfigure}
  \caption{\label{fig:mnistvsmnistm}A benchmark for domain adaptation: MNIST$\rightarrow$MNIST-M~\citep{ganin2015unsupervised}.}
\end{figure}

We revisit previous empirical results in light of our theoretical findings with emphasis on Domain-Adversarial Neural Networks (DANN) by~\citep{ganin2016domain}\footnote{Our implementation is based on that of  \url{https://github.com/pumpikano/tf-dann}}.

\begin{figure*}[t!]
  \centering
  \begin{subfigure}{.35\textwidth}
    \centering
    \includegraphics[width=\textwidth]{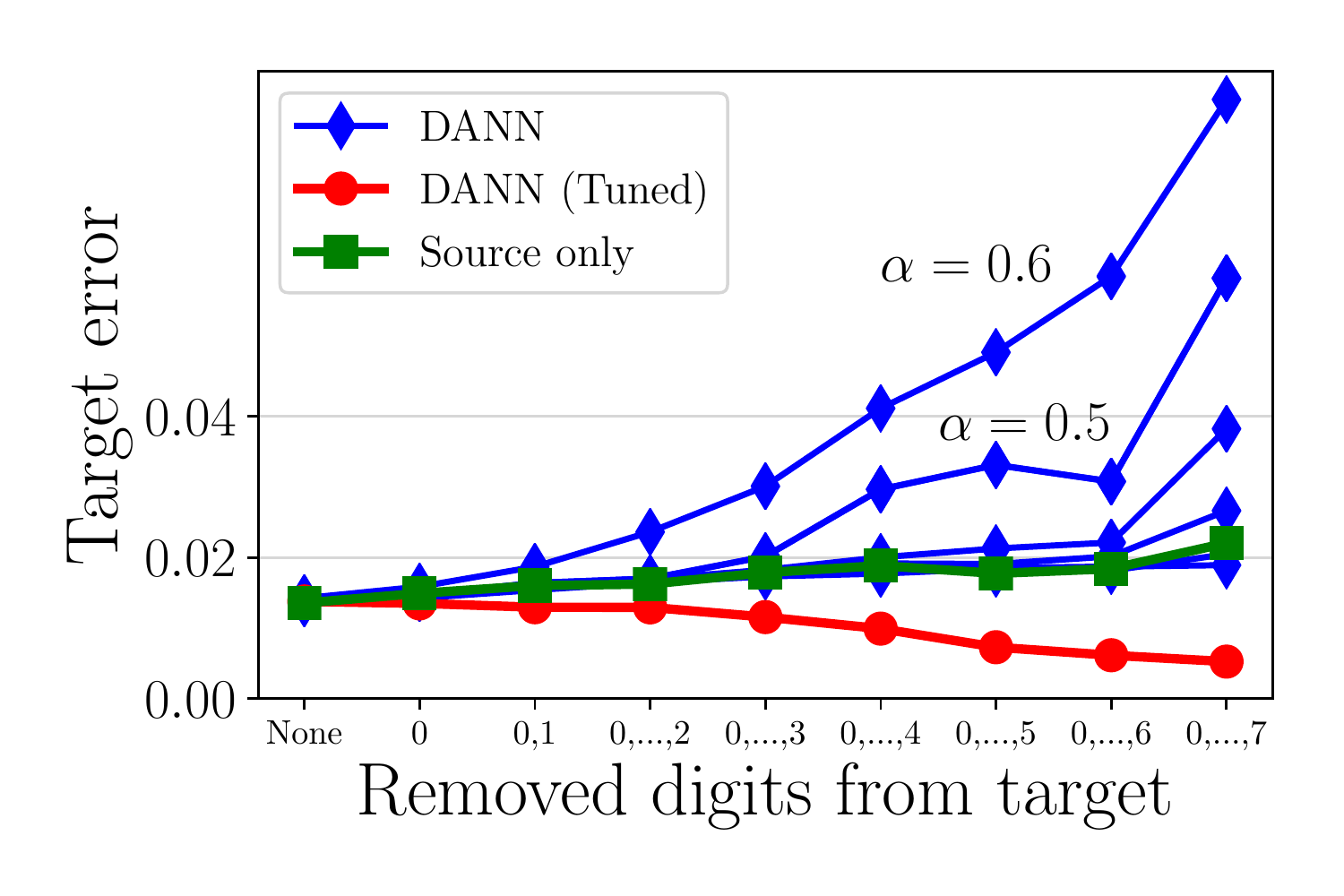}
    \caption{MNIST$\rightarrow$MNIST}
  \end{subfigure}
  \begin{subfigure}{.35\textwidth}
    \centering
    \includegraphics[width=\textwidth]{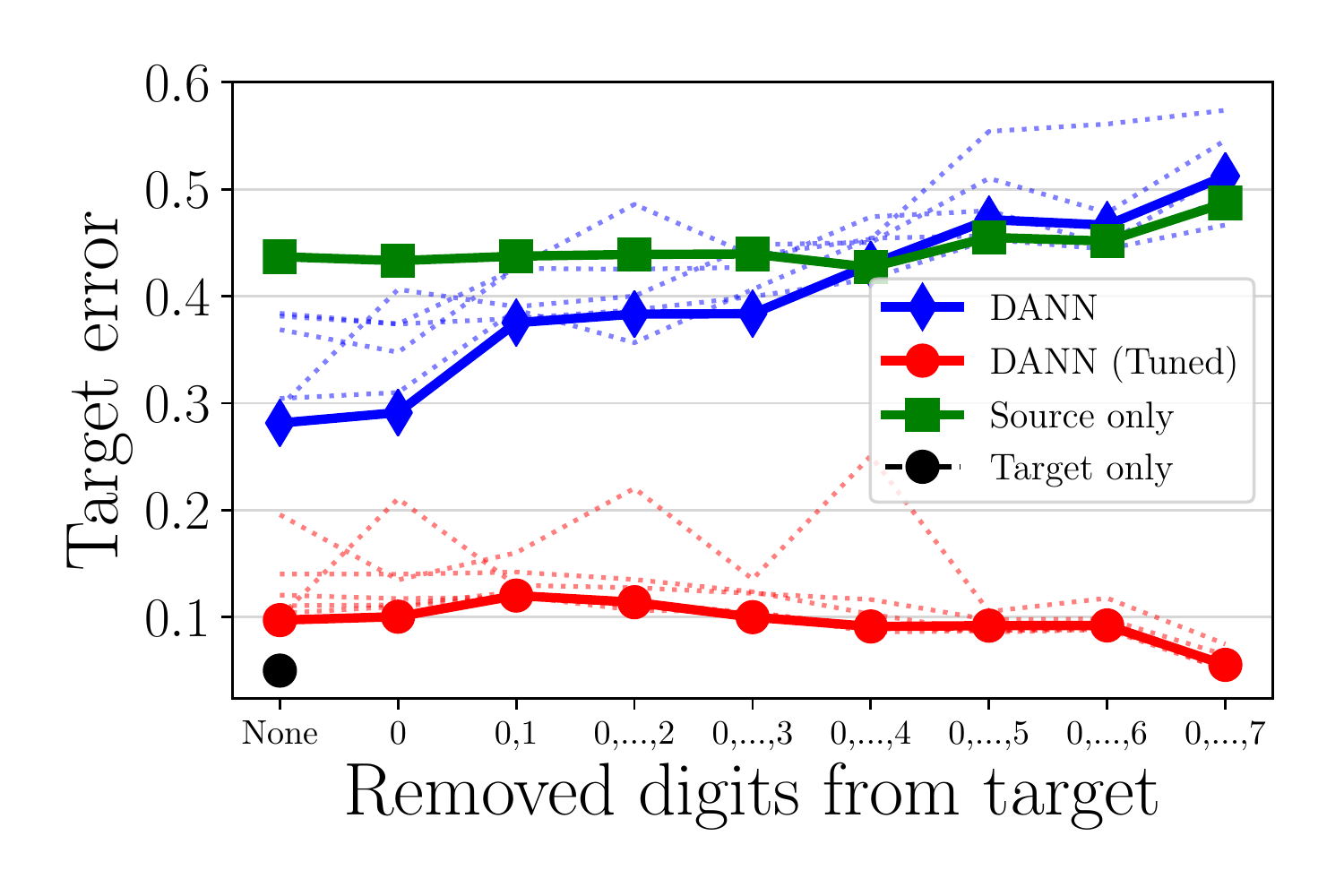}
    \caption{MNIST$\rightarrow$MNIST-M}
  \end{subfigure}
  \begin{subfigure}{.28\textwidth}
    \centering
    \vspace{.5em}
    \includegraphics[width=\textwidth]{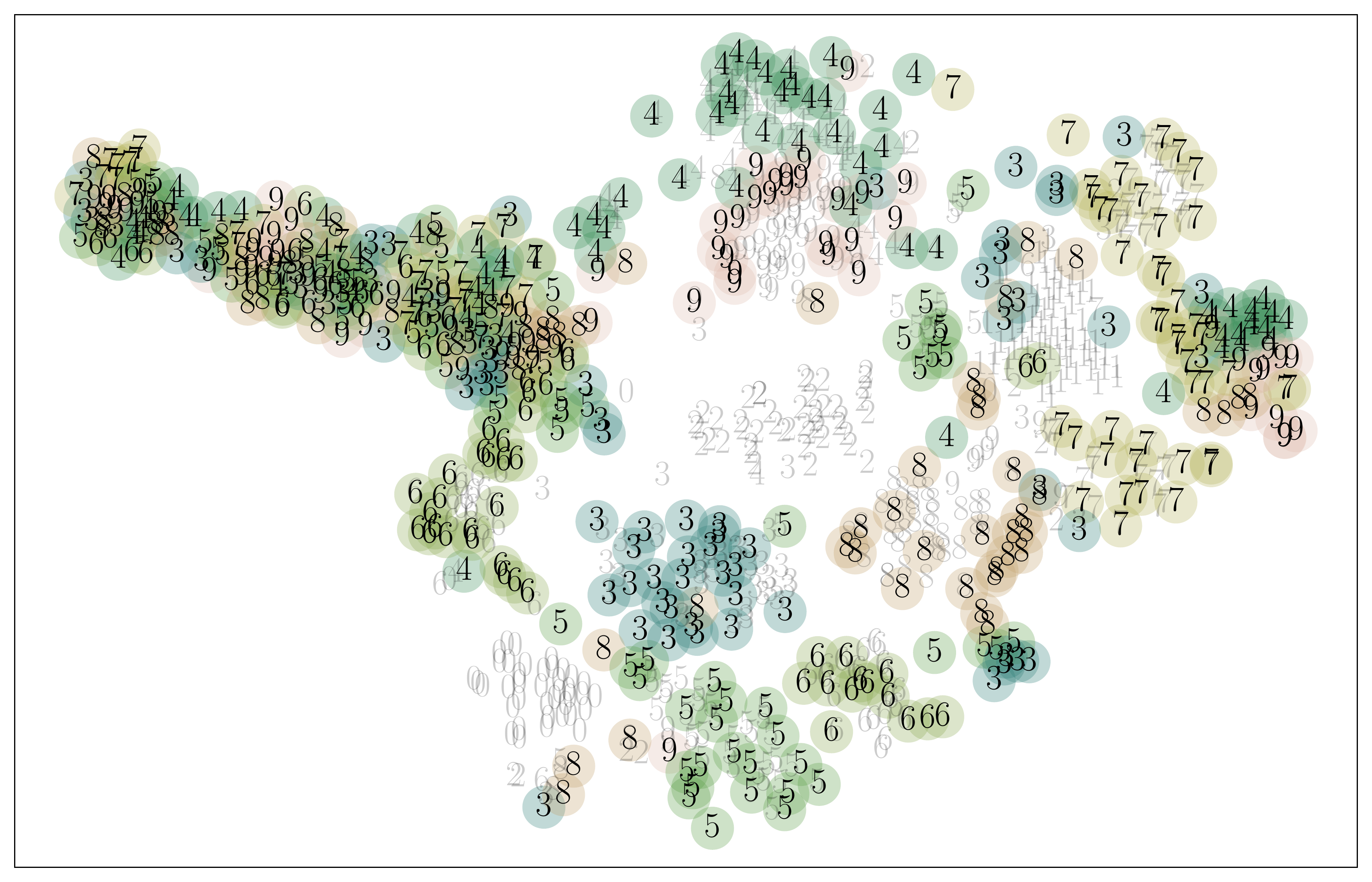}
    \vspace{.5em}
    \caption{\label{fig:embeddings0}MNIST$\rightarrow$MNIST-M$\setminus\{0,1,2\}$}
  \end{subfigure}
  \caption{\label{fig:labelshift}Left: Target error as a function of marginal label distribution. For each setting, a DANN model is trained on unlabeled target data and labeled source data. We compare the accuracy of this model to a model tuned on target labels but with a fixed representation given by the first model. Different lines of the same color indicate different values of penalty strength $\alpha \in \{0.0, ..., 0.6\}$. Right: Embeddings learned by DANN with equal (top) and unequal (bottom) label marginal distributions. In MNIST-M$\setminus\{0,1,2\}$, all images of digits 0,1,2 have been removed. Grey digits are from the source domain and black digits from the target domain. }
  %\fxfatal{Check caption. c) hard to read in b\&w}
\end{figure*}

\subsection{Plausibility of sufficient assumptions}
The most common benchmarks for domain adaptation algorithms are computer vision and natural language processing tasks. One example is the $\text{MNIST}\rightarrow\text{MNIST-M}$ task~\citep{ganin2015unsupervised}, in which the goal is to learn to classify handwritten digits overlayed with random photographs (MNIST-M) based on labeled images of digits alone (MNIST)~\citep{lecun1998gradient} (see Figure~\ref{fig:mnistvsmnistm}). For this task, we can immediately rule out Assumption~\ref{asmp:supp} of sufficient support, as MNIST-M images are full-color images that have measure 0 in MNIST. Still, previous work have achieved target accuracy of $>55\%$ when training on source data alone, and $>80\%$ when using unlabeled target data, compared to $>95\%$ when using labeled target data~\citep{ganin2015unsupervised,bousmalis2016domain}.
These results support Assumption~\ref{asmp:domaininv}---that there exists a  domain-invariant representation in which the labeling function is approximately realizable.

%
%In conclusion, no assumptions known to be sufficient for identification hold in benchmark tasks, yet empirical results are promising.

\subsection{Contrasting support and domain variance}
\label{sec:exp_invariance}
When label marginal distributions differ under covariate shift, $\ps(Y) \neq \pt(Y), \ps(Y\mid X) = \pt(Y\mid X)$, such as when objects of a certain class appear more often in one domain, a distance between feature marginals, $\ps(X), \pt(X)$, is induced. Authors have studied this restricted setting in detail~\citep{zhang2013domain,lipton2018detecting}. If additionally the target domain is made up of a subset of the source domain, encouraging domain invariance may cause more harm than good. We study a) the performance of DANN models under domain shift with sufficient support, and b) the realizability of the label  in the learned representation.

We create a task in which the source domain is the standard MNIST dataset and the target domain is a version of MNIST for which domain shift is induced by successively removing digit classes from the support of the target domain, leaving the source domain fixed. In this setup, the support of the target domain is contained in the source domain, and empirical risk minimization based on source data alone should be a good baseline. We compare to the case where the target is replaced by MNIST-M, but perturbed in the same way.

%We consider the MNIST$\rightarrow$MNIST-M task, the standard version of which contains one perturbed image in MNIST-M for every image in the MNIST dataset---the two datasets have the same label marginal distribution. To introduce label marginal shift, we successively remove digit classes from the support of the target domain, leaving the source domain fixed.

The DANN model optimizes \eqref{eq:discmin}, with $d$ an adversarial neural network classifying images by domain, and $\alpha$ a hyperparameter  controlling the strength of this penalty in the objective
$
O = 2(1-|\alpha-0.5|)((1-\alpha) \hat{R}_s + \alpha d)~.
$
In this way, we interpolate between empirical risk minimization ($\alpha=0$), the standard DANN formulation ($\alpha=0.5$) and prioritizing domain invariance ($\alpha > 0.5$). We compare the error of two different models: 1) The standard DANN estimator $h_{dann} = f_{dann} \circ \phi_{dann}$, and 2) A model $h_{tuned} = f_{tuned} \circ \phi_{dann}$ in which the learned representation $\phi_{dann}$ from 1) is fixed and the prediction function $f$ is fit to the \emph{target labels} (Tuned). The latter serves to give an upper bound on best-case risk when predicting from the representations learned by DANN.

In Figure~\ref{fig:labelshift}, we observe that models trained without target supervision (DANN) perform steadily worse on MNIST$\rightarrow$MNIST, the more the label marginal distribution is perturbed. This holds also for MNIST$\rightarrow$MNIST-M, where sufficient support is not satisfied. There,  DANN is beneficial for small label shift, but eventually does no better than a model trained using only source data.
Learning with a domain-adversarial loss appears to have little impact on the \emph{realizability} of the target label in the representation; the target-tuned models achieve almost as good performance as the fully target-trained lower bound. In Figure~\ref{fig:embeddings0}, we see that the embeddings learned using DANN models under label marginal shift show worse separation between classes, than the embeddings learned under equal label marginal distributions (see Appendix~\ref{app:experiments}).

\section{Discussion}
\label{sec:discussion}

We have studied algorithms for unsupervised domain adaptation based on domain-invariant representation learning and the theoretical arguments used to support them. We find that, despite empirical success, the theoretical justification of these algorithms is flawed in that oft-cited generalization bounds are not minimized by the learned representations. In particular, the literature has failed to characterize conditions under which domain-invariant representations lead to consistent estimation. We have found through examples and experiments on domain adaptation benchmarks that domain invariance is often too strong a requirement for learning, both when there is overlap between domains and when there is not. This stems from the fact that overlapping support is sufficient for domain transfer, and equality in densities is not necessary.

We have proposed alternative bounds that measure distance in support instead of density and that explicitly recognize loss incurred by non-invertible representations. Our bounds suggest several ways to design new algorithms. First, minimizing the second term in our bound, the support sufficiency divergence, may be achieved by replacing indicator functions by hinge losses (see Appendix~\ref{app:model}). This increases the looseness of the bound, but makes its derivative informative. In the same spirit, we may design new heuristics that regularize representations only in points at which the source density is much smaller than the target density.
Second, while the excess adaptation error induced by learning non-invertible transformations is unobservable, it is associated with the information loss of the representation. To avoid this, we may attempt to maintain a small excess by imposing a reconstruction loss on the representation, similar to~\citet{bousmalis2016domain}.

\section*{Acknowledgements}
We thank Zach Lipton, Alexander D'Amour, Christina X Ji and Hunter Lang for insightful feedback. This work was supported in part by Office of Naval Research Award
No. N00014-17-1-2791 and the MIT-IBM Watson AI Lab.

\bibliographystyle{jmb}
\bibliography{da}

\appendix
%
% Appendix
%

\section{Proofs}

\subsection{Proof of bounds for support sufficiency divergence}
\label{app:suppbounds}

\begin{thmlem}
  The support sufficiency divergence is bounded with $0 \leq \dsuppeps{p}{q}   \leq 1$, and the bounds are tight.
\end{thmlem}
\begin{proof}
  The lower bound holds and is tight because
  \begin{align*}
    \dsuppeps{p}{q} = \int_x (q(x) - p(x))\delta_{p,q}(x) \\
    = \int_x \max(q(x) - p(x), 0)\mathds{1}[p(x) \leq \epsilon]\mathds{1}[p(x) \leq \epsilon]~.
  \end{align*}
  which is clearly non-negative. Moreover, for $\epsilon \leq \inf_x p(x)$, $\dsuppeps{p}{q} = 0$. The upper bound holds trivially as $\delta_{p,q}(x) \leq 1$. For tightness, let $q,p$ be discrete densities over two states, $q = [1. , 0.]$ and $p = [0., 1.]$. Then with $\epsilon > 0$, $\dsuppeps{p}{q} = 1$.
\end{proof}

Recall that
\begin{equation}
w_{p,q}^\epsilon(x) = \left\{
\begin{array}{ll}
  q(x)/p(x) & \mbox{ if } p(x) \geq \epsilon \\
  1 & \mbox{ otherwise }
\end{array}
\right.
\end{equation}

\subsection{Proof of Lemma~\ref{lem:overlap}}
\label{app:proof_lem_overlap}
\begin{thmlem}
Let $p,q$ be densities over $\cX$. Further, define $\delta_{p,q}(x) = \mathds{1}[p(x)\leq \epsilon \;\mbox{ and }\; p(x)\leq q(x)]$. Then,
\begin{align*}
\En_q[f(x)] & \leq \En_p\left[w_{p,q}^\epsilon(x) f(x)\right] \\
& + M \cdot \left( \En_q[\delta^\epsilon_{p,q}(x)] - \En_p[\delta^\epsilon_{p,q}(x)]\right)
\end{align*}
\end{thmlem}
\begin{proof}
We have,
\begin{align*}
&\E_q[f(x)] = \int_x q(x)f(x) dx  \\
&= \int_{x : p(x)>\epsilon} q(x)f(x) dx + \int_{x : p(x)\leq \epsilon} q(x)f(x) dx \\
&= \int_{x : p(x)>\epsilon} \frac{q(x)}{p(x)} p(x)f(x) dx + \int_{x : p(x)\leq \epsilon} q(x)f(x) dx \\
&\leq \En_p\left[w_{p,q}^\epsilon(x) f(x)\right] + \int_{x : p(x)\leq \epsilon} (q(x)-p(x))f(x) dx \\
&\leq \En_p\left[w_{p,q}^\epsilon(x) f(x)\right] + M \int_{\substack{x : p(x)\leq \epsilon \\ p(x) \leq q(x)}} (q(x)-p(x)) dx  \\
&= \En_p\left[w_{p,q}^\epsilon(x) f(x)\right] \\
&+ M \int_{x} (q(x)-p(x)) \underbrace{\mathds{1}[p(x)\leq \epsilon \land p(x)\leq q(x)]}_{\delta^\epsilon_{p,q}(x)} dx\\
& = \En_p\left[w_{p,q}^\epsilon(x) f(x)\right] + M \cdot \left( \E_q[\delta_{p,q}(x)] - \E_p[\delta^\epsilon_{p,q}(x)]\right)
\end{align*}
Further, $p=q$ implies equality when $\epsilon \geq \sup_x q(x)$.
\end{proof}

\subsection{Proof of Theorem~\ref{thm:repbound}}
\label{app:proof_repbound}

\begin{thmlem} Assume that $\pt(Y \mid X) = \ps(Y\mid X)$. Define $Z = \phi(X)$ and let $h(x) = f(\phi(x))$. Then,
$$
\E_{x,y \sim q(x,y)}[\ell(h(x),y)] = \E_{z,y \sim q(z,y)}[\ell(f(z),y)]
$$
\end{thmlem}
\begin{proof}
\begin{align*}
& \E_{z,y \sim q(z,y)}[\ell(f(z),y)] = \\
& = \int_{z,y} q(z, y) \ell(f(z), y) dzdy \\
& = \int_{z,y} \ell(f(z), y) \int_{x \in \phi^{-1}(z)} q(x, y) dxdzdy \\
& = \int_{x,y} q(x, y) \int_{z} \mathds{1}[z = \phi(x)]\ell(f(z), y) dzdxdy \\
& = \int_{x,y} q(x, y) \ell(h(x), y) dxdy \\
& = \E_{x,y \sim q(x,y)}[\ell(h(x),y)]
\end{align*}
\label{lem:helper1}
\end{proof}

\begin{thmlem}
$$
R_t(h) = E_{q(z)p(y\mid z)}[\ell(f(z),y)] + \eta_\phi^\ell(f,y)
$$
\label{lem:helper2}
\end{thmlem}
\begin{proof}
By Lemma~\ref{lem:helper1}
$$
R_t(h) = E_{q(x,y)}[\ell(h(x),y)] = E_{q(z,y)}[\ell(f(z),y)]
$$
We have that
\begin{align*}
& E_{q(z)q(y\mid z)}[\ell(f(z),y)] \\
& = \int_{z}\int_{y} q(z)q(y\mid z)\ell(f(z),y) dy dz \\
& = \int_{z}\int_{y} \left( \int_{x : \phi(x) = z} q(x)dx \right) q(y\mid z)\ell(f(z),y) dy dz \\
& = \int_{z}\int_{y} \int_{x \in \phi^{-1}(z)} q(x) q(y\mid \phi(x))\ell(h(\phi(x)),y) dx dy dz \\
& = \int_{x} \int_{y} q(x) q(y\mid \phi(x))\ell(h(\phi(x)),y) dx dy \\
& = \int_{x} \int_{y} q(x) \ell(h(\phi(x)),y) \big[ \underbrace{q(y\mid x)}_{=p(y\mid x) \text{ (assmp.)}} \\ 
& + q(y\mid \phi(x)) -  q(y\mid x) \big] dx dy
\end{align*}
and by the same argument,
\begin{align*}
& E_{q(z)p(y\mid z)}[\ell(f(z),y)] \\
& = \int_{x} \int_{y} q(x) \ell(h(\phi(x)),y) \big[p(y\mid x) \\
& +  p(y\mid \phi(x)) -  p(y\mid x) \big] dx dy
\end{align*}
and as a result,
\begin{align*}
& E_{q(z)q(y\mid z)}[\ell(f(z),y)] - E_{q(z)p(y\mid z)}[\ell(f(z),y)] \\
&\;\;\;\; = \E_q(x)[ \E_{q(y\mid \phi(x))} \ell(f(\phi(x)),y)] - \E_{q(y\mid x)} \ell(f(\phi(x)),y)] \\
&\;\;\;\; + \E_q(x)[ \E_{p(y\mid \phi(x))} \ell(f(\phi(x)),y)] - \E_{p(y\mid x)} \ell(f(\phi(x)),y)] \\
&\;\;\;\; = \eta_\phi^\ell(f,y)
\end{align*}
\end{proof}

\begin{reptheorem}{thm:repbound}
Consider any feature representation $z = \phi(x)$ with $\phi : \cX \rightarrow \cZ$ and prediction function $f : \cZ \rightarrow \cY$, and define $h = f \circ \phi$. Further, let $p(Z)$ and $q(Z)$ be the two distributions induced by the representation $\phi$ applied to $X$ distributed according to $p(X), q(X)$.  Further, assume that for any hypothesis $h$ and a loss function $\ell$, $\sup_{x\in \cX, h\in \cH} \ell(h(x),y) \leq M$. Now, with $\epsilon > 0$, we have the following result.
\begin{align*}
R_q(h) & \leq \En_{p}\left[w_{\ps,\pt}^\epsilon(z) \ell(f(z), y)\right] \\
& + M\dsuppeps{p(z)}{q(z)}  + \eta^\ell_\phi(f, y)~.
\end{align*}
\end{reptheorem}
\begin{proof}
By Lemma~\ref{lem:helper2}, we have that
$$
R_q(h) \leq E_{q(z)p(y\mid z)}[\ell(f(z),y)]  + \eta^\ell_\phi(f, y)~.
$$
Further,
\begin{align*}
& E_{q(z)p(y\mid z)}[\ell(f(z),y)] \\
& = \iint_{z \in \cZ, y \in \cY} q(z)p(y\mid z) \ell(f(z),y) dy dz \\
& = \iint_{z : p(z) \geq \epsilon, y \in \cY} q(z)p(y\mid z) \ell(f(z),y) dy dz \\
& + \iint_{z : p(z) < \epsilon, y \in \cY} q(z)p(y\mid z) \ell(f(z),y) dy dz \\
& = \iint_{z \in \cZ, y \in \cY} w_{\ps,\pt}^\epsilon(z) p(z) p(y\mid z) \ell(f(z),y) dy dz \\
& + \int_{z : \substack{p(z) < \epsilon \\ p(z) \leq q(z)}} \underbrace{(q(z) - p(z))}_{\geq 0} \underbrace{\int_{y}p(y\mid z) \ell(f(z),y) dy}_{\in [0, M]} dz \\
& + \underbrace{\int_{z : \substack{p(z) < \epsilon \\ p(z) > q(z)}} (q(z) - p(z)) \int_{y}p(y\mid z) \ell(f(z),y) dy dz}_{\leq 0}~.
\end{align*}
\end{proof}

\subsection{Kernel support divergence}

Theorem~\ref{thm:repboundk} may be viewed as a measuring differences in density only where supports differ significantly. In the case where $\cL$ is a Hilbert space, similar to the maximum mean discrepancy~\citep{gretton2012kernel}, we may decompose $\dsuppepsk{p}{q}{\cL}$ using reproducing kernels.

\begin{thmlem}\label{lem:kernel}
Let $\cH$ be the reproducing-kernel Hilbert space with kernel $k : \cX \times \cX \rightarrow \mathbb{R}$. Then,
\begin{align}
& \dsuppeps{p}{q}_\cG = \E_{x,x' \sim p}[\delta^\epsilon_p(x,x')  k(x,x') \\
& - 2\E_{x \sim p, x' \sim q}[\delta^\epsilon_p(x,x')  k(x,x')]
 \nonumber \\
& + \E_{x,x' \sim q}[\delta^\epsilon_p(x,x')  k(x,x')] \nonumber
\end{align}
\end{thmlem}
\begin{proof}
The proof follows from \citet{gretton2012kernel} and can be found in Appendix~\ref{app:proof_kernel}.
\end{proof}

\label{app:proof_kernel}
Let $\delta^\epsilon_p(x) = \mathds{1}[p(x)\leq \epsilon]$ and $\delta^\epsilon_p(x,x') = \delta^\epsilon_p(x)\delta^\epsilon_p(x')$
\begin{align}
\dsuppeps{p}{q}_\cG & \coloneqq \sup_{g \in \cG} \left| \E_{q}[\delta^\epsilon_p(x)g(x)] - \E_{p}[\delta^\epsilon_p(x) g(x)] \right| \nonumber \\
& = \E_{x,x' \sim p}[\delta^\epsilon_p(x,x')  k(x,x')] \nonumber \\
& - 2\E_{x \sim p, x' \sim q}[\delta^\epsilon_p(x,x')  k(x,x')] \\
& + \E_{x,x' \sim q}[\delta^\epsilon_p(x,x')  k(x,x')] \nonumber
\end{align}
\begin{proof}
  Follow \url{http://alex.smola.org/teaching/iconip2006/iconip_3.pdf} page 18--20.
\end{proof}

\section{Model}
\label{app:model}

We may bound $\dsuppeps{p}{q}$ using the hinge loss as follows,
\begin{align*}
&\dsuppeps{p}{q} \\
& \leq \E_{x\sim q}\left[\max\left(0,2-\frac{p(x)}{\epsilon} \right)\max\left(0,2-\frac{p(x)}{q(x)} \right)\right] \\
& - \E_{x\sim p}\left[\max\left(0,1-\frac{p(x)}{\epsilon}\right)\max\left(0,1-\frac{p(x)}{q(x)} \right)\right] \\
& =: \dsuppx{\tilde{d}^\epsilon}{p}{q}~.
\end{align*}

\section{Experiments}
\label{app:experiments}
In Figure~\ref{fig:embeddings}, we see that the embeddings learned using DANN models under label marginal shift show worse separation between classes, than the embeddings learned under equal label marginal distributions.

\begin{figure}[tbp!]
  \centering
  \begin{subfigure}{.9\columnwidth}
    \centering
    \includegraphics[width=\textwidth]{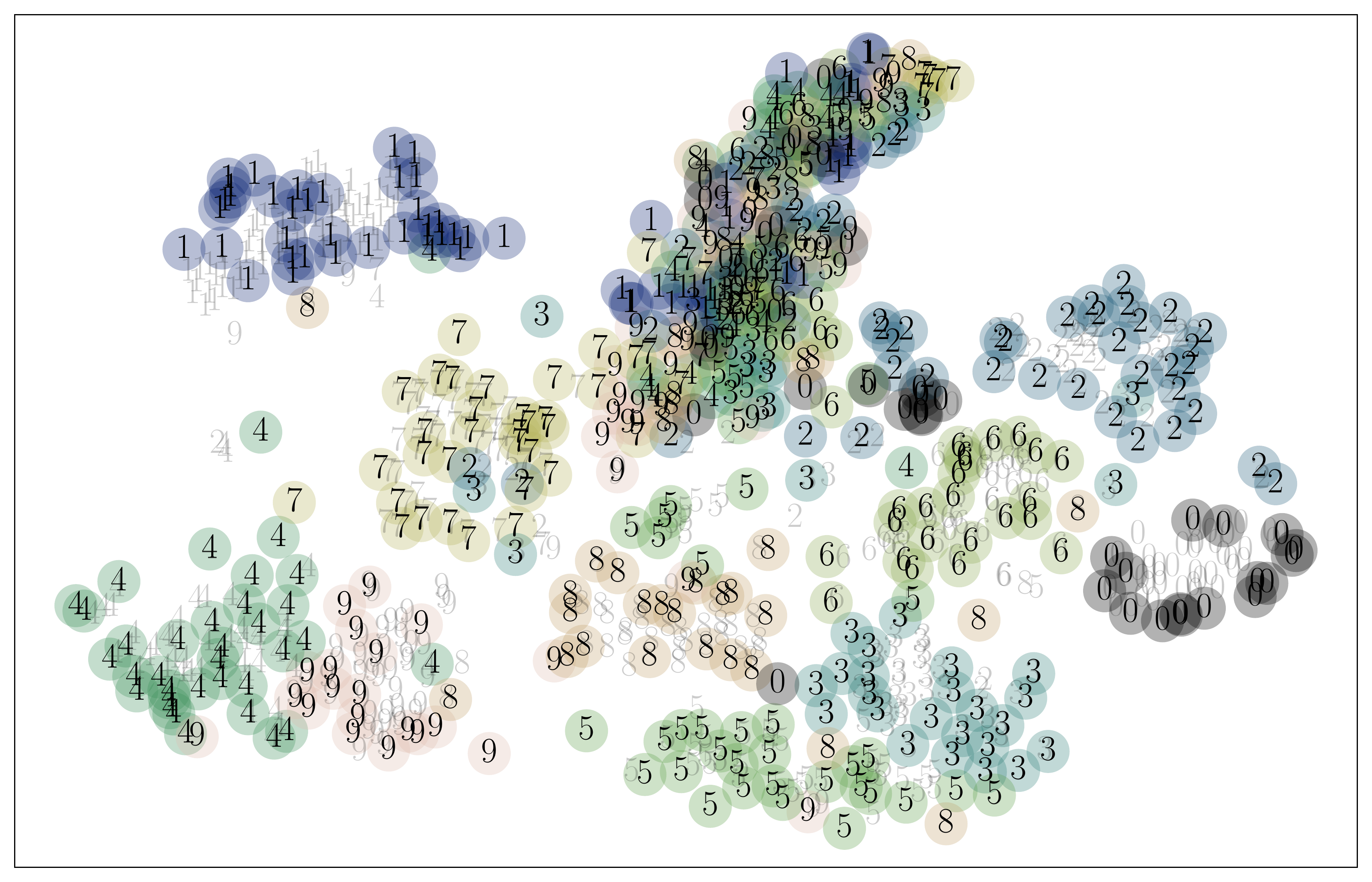}
    \caption{MNIST$\rightarrow$MNIST-M}
  \end{subfigure}
  \begin{subfigure}{.9\columnwidth}
    \centering
    \includegraphics[width=\textwidth]{mnist_notarget012_embed.pdf}
    \caption{MNIST$\rightarrow$MNIST-M$\setminus\{0,1,2\}$}
  \end{subfigure}
  \caption{\label{fig:embeddings}Embeddings learned by DANN with equal (top) and unequal (bottom) label marginal distributions. In MNIST-M$\setminus\{0,1,2\}$, all images of digits 0,1,2 have been removed. Grey digits are embeddings and labels of the source domain. Black digits against colored background are from the target domian. }
\end{figure}

\section{Consistent domain-invariant variable selection}
\label{app:consistency}
Consider a matrix $A \in \{0, 1\}^{k \times d}$ with $k<d$ such that $\forall j : \sum_{i=1}a_{ij} \leq 1$ and $\forall i : \sum_{j=1}a_{ij} \leq 1$. In other words, $A$ is a variable selection operator on $X$. Now, assume that $Z \coloneqq \Phi(X) \coloneqq AX$ is sufficient for $Y$ on $\ps$ and $\pt$ and  that $\ps(AX) = \pt(AX)$. Further, assume labeled data is observed under $p$ and unlabeled data observed under $\pt$. Then, is $Y$ identifiable based on domain-invariance and source predictive loss?

\begin{thmcond}[Smoothness]\label{cond:smoothness}
With $\Sigma_L = \{f: \sum_{k\in \mathbb{Z}^d} k_2^2\langle f, \varphi_k\rangle^2 \leq L; \forall j \in \{1, \ldots, d\}$ for $L>0$, $f$ is $L$-smooth if $f\in \Sigma_L$, with $\varphi_k$ the trigonometric fourier basis.
\end{thmcond}
\begin{thmcond}[Identifiability]
A sufficient set of variables $J$, such that $\exists \bar{f} : f(x) = \bar{f}(x_J)$ for all $x\in \mathbb{R}^d$, is $\kappa$-identifiable if for all $j\in J$,
$$
\int_{[0,1]^d}(f(x) - \int_0^1 f(x) dx_j)^2 dx \geq \kappa~.
$$
\end{thmcond}
\begin{thmcond}[Positive bounded support]
The density $\ps(x)$ has positive bounded support over $[0,1]^d$ if with ${\ps}_{\min} > 0$, $\forall x\in [0,1]^d : \ps(x) \geq {\ps}_{\min}$ and $\forall x \notin [0,1]^d : \ps(x) = 0$.
\end{thmcond}
\begin{thmcond}[Bounded $\infty$-norm and $2$-norm]
A function $f$ has bounded $\infty$-norm and $2$-norm with respect to $p$ if $\pr_{X\sim \ps}(|f(X)| \leq L_\infty) = 1$ and $\E_{X\sim \ps}[f(X)^2] \leq L_2^2$.
\end{thmcond}
\begin{thmcond}[Sub-gaussian additive noise]\label{cond:subgaussian}
The observed outcome may be written as $Y_i = f(x_i) + \sigma\epsilon_i$ with $\E[e^{t\epsilon_i}\mid X_i] \leq e^{t^2/2}$ for all $t>0$.
\end{thmcond}

\begin{thmthm}[Variable selection in non-parametric regression~\citep{comminges2012tight}]\label{thm:varsel_regression} Assume that Conditions~\ref{cond:smoothness}--\ref{cond:subgaussian} hold, with known parameters ${\ps}_{\min}$, $\theta = 2L/\kappa$ and $L_2$. Then, there is an estimator $\hat{J}$ that satisfies $\pr(\hat{J} \neq J) \leq (8d/d^*)^{-d^*}$.
\end{thmthm}

\citet{comminges2012tight} give a constructive proof of Theorem~\ref{thm:varsel_regression} in which the chosen estimator is allowed to depend on the density $\ps(x)$.

%-------------------------------------- REFERENCES ----------------------------------------------%

%-------------------------------------- END OF DOCUMENT ----------------------------------------------%

\end{document}